\documentclass{article}
\usepackage{ijcai23}

\usepackage{times}
\usepackage{soul}
\usepackage{url}
\usepackage[hidelinks]{hyperref}
\usepackage[utf8]{inputenc}
\usepackage[small]{caption}
\usepackage{graphicx}
\usepackage{amsmath}
\usepackage{amsthm}
\usepackage{booktabs}
\usepackage{algorithm}
\usepackage{algorithmic}
\usepackage[switch]{lineno}

\usepackage{pgfplots}
\usepackage{multirow}
\usepackage{multicol}
\usepackage{subfigure}
\usepackage{mathrsfs}
\usepackage{amssymb}
\definecolor{color1}{RGB}{52,115,177}
\definecolor{color2}{RGB}{123,173,208}
\definecolor{color3}{RGB}{204,220,242}
\definecolor{color4}{RGB}{239,179,101}
\definecolor{color5}{RGB}{226,141,79}
\definecolor{color6}{RGB}{187,44,38}

\newtheorem{theorem}{Theorem}
\newtheorem{proposition}{Proposition}

\title{Understanding and Improving Deep Graph Neural Networks: \\ A Probabilistic Graphical Model Perspective}

\author{
Jiayuan Chen$^1$
\and
Xiang Zhang$^{2}$\and
Yinfei Xu$^{2,*}$\and
Tianli Zhao$^3$ \and
Renjie Xie$^{4,5}$\And
Wei Xu$^{4,5}$
\affiliations
$^1$School of Computer Science and Engineering, Southeast University\\
$^2$School of Information Science and Engineering, Southeast University\\
$^3$Institute of Automation, Chinese Academy of Sciences\\
$^4$National Mobile Communications Research Laboratory, Southeast University\\
$^5$Purple Mountain Laboratories\\
}

\begin{document}

\maketitle

\begin{abstract}
Recently, graph-based models designed for downstream tasks have significantly advanced research on graph neural networks (GNNs). GNN baselines based on neural message-passing mechanisms such as GCN and GAT perform worse as the network deepens. Therefore, numerous GNN variants have been proposed to tackle this performance degradation problem, including many deep GNNs. However, a unified framework is still lacking to connect these existing models and interpret their effectiveness at a high level. In this work, we focus on deep GNNs and propose a novel view for understanding them. We establish a theoretical framework via inference on a probabilistic graphical model. Given the fixed point equation (FPE) derived from the variational inference on the Markov random fields, the deep GNNs, including JKNet, GCNII, DGCN, and the classical GNNs, such as GCN, GAT, and APPNP, can be regarded as different approximations of the FPE. Moreover, given this framework, more accurate approximations of FPE are brought, guiding us to design a more powerful GNN: coupling graph neural network (CoGNet). Extensive experiments are carried out on citation networks and natural language processing downstream tasks. The results demonstrate that the CoGNet outperforms the SOTA models.
\end{abstract}

\section{Introduction}
 Graphs have received increasing attention in recent years, driving the unprecedented development of graph neural networks (GNNs). In particular, neural graph-based models are applied to various downstream tasks such as biology \cite{bio2,bio1}, social analysis \cite{soc2}, and computer vision \cite{stgcn,resgcn,vig}. However, it has been observed that classical GNNs like GCN and GAT achieve the best performance in most downstream tasks at only a two-layer network. With the network depth increasing, the performance drops rapidly. 
 
 Hence, a profusion of work tries to tackle the performance degradation problem. These deep GNN models design various feature propagation processes based on the message-passing mechanism of GNN. For instance, JKNet \cite{jknet} flexibly adjusts the neighbor range of each node by using a layer aggregation network architecture. Another excellent model Deep Adaptive Graph Neural Network decouples transformation and propagation to leverage larger receptive fields. GCNII \cite{gcnii} adds Initial residual and Identity mapping based on GCN, extending GCN to a $k$-th order of polynomial filter. Despite fruitful achievements, these works study GNNs from various perspectives, including the spatial and spectral domains that have garnered the most attention in recent years. Driven by the different interpretations of GNNs, a question arises: Is there a framework that can interpret these deep GNNs uniformly? If such a framework exists, we can leverage it to understand the performance degradation problem of classical GNNs more in-depth and identify the weakness of the existing deep GNNs to facilitate their development.
 
In this paper, we understand graph convolution from a probabilistic perspective and develop a general interpretation framework for deep GNNs.  GNNs and probabilistic graphical models (PGMs) are prominent tools for studying structured data.  \cite{dai} note the iterative propagation scheme of variational inference in PGMs, which provides an intuitively appealing direction for connecting GNNs. Thus, we start from the variational inference of Markov random fields and obtain its iterative solution embedded in the Hilbert space. Inspired by the close connection between GNNs and the first-order Taylor expansion of the iterative solution, we establish a unified framework for representing different GNNs and their deep variants. We then incorporate some popular deep graph networks, such as JKNet, DGCN, and GCNII \cite{jknet,argue,gcnii}, into the developed theoretical framework with the help of high-order Taylor approximation. Furthermore, with the proposed framework, we analyze the performance degradation problem and design a more effective model, the \underline{Co}upling \underline{G}raph Neural \underline{Net}work (CoGNet). Through extensive experiments on three semi-supervised node classification benchmarks and natural language processing downstream tasks, including text classification and multi-hop question answering, we demonstrate that the CoGNet achieves strong performances in both shallow and deep network layers, including multiple SOTA results.

\section{Related Work}
\subsection{Graph Neural Networks}
Graph Neural Network (GNN) is a framework for graph representation learning. The basic GNNs follow a neural message-passing mechanism. They exchange messages between connected nodes and update node features using neural networks. In the past few years, some baseline GNN models, such as GCN, GIN, and GAT  \cite{gcn,xu2018powerful,gat}, have been proposed and achieved great popularity. These models have been successfully applied to many downstream tasks, including recommender systems \cite{rm,lightgcn}, drug analysis \cite{drug1}, community detection \cite{cd1,cd2,cd3}, etc. 

However, these GNN models are known to suffer from performance degradation when deepening the network structures. This phenomenon may be caused by problems such as over-parameterization, over-smoothing, and weak generalization of the model. Many works on improving the performance of deep GNN models have been proposed recently.  Jumping
Knowledge Network \cite{jknet} is one of the earliest deep GNNs, combining the node representations of each layer at the last layer.  Another excellent model DAGNN \cite{dagnn} adaptively incorporates information
from large receptive fields. Recently, DropEdge \cite{dropedge} randomly mutes a certain number of edges during each training epoch to alleviate the over-smoothing issue. Other deep models, such as GCNII \cite{gcnii} and DGCN \cite{argue}, have also achieved good performance. However, these models design their propagations from different perspectives, and a unified language is lacking to describe them.

\subsection{Interpretability of GNNs}
Theoretical interpretability has become a focus of research on graph representation learning. As a generalization of convolution on graph-structured data, graph convolution has inspired many works to explore GNNs from the spectral perspective. \cite{li2018deeper} shows that the propagation process of the GCN is a low-pass filter and attributes performance degradation to the over-smoothing problem. 
\cite{ma2021unified} understands GNNs as graph signal denoising. Graph wavelet neural network \cite{graphwl} uses graph wavelets to replace the Fourier basis in GCN. \cite{Chebyshev,sgc,Cayleynets} design the network aggregation from the perspective of spectral filters. Another line of work analyzes the connections between GNNs and graph isomorphism testing. \cite{xu2018powerful} proves that message-passing GNNs are no more powerful than the Weisfeiler-Lehman (WL) algorithm \cite{wltest} when given discrete information as node features. \cite{khop} shows that the expressive
power of $K$-hop message passing is bounded by 3-WL test.
Most recently, several works have been devoted to generalizing deep learning theory to graph machine learning. \cite{graphntk} uses Neural Tangent Kernel to analyze infinitely wide multi-layer GNNs trained by gradient descent.  \cite{esser2021learning} explores
generalization in GNNs through VC Dimension based generalization error bounds. In this work, we understand GNNs from a novel probabilistic perspective and analyze their intrinsic connections.

\section{GNNs as Embedded Graphical Models}
\label{xxx}
In this section, we establish the connections between the GNN models and the probabilistic inference in graphical models. 
We first introduce the notations used throughout the paper. 

\noindent \textbf{Notations.} Given a weighted undirected Graph $G = (V,E)$ with $n$ nodes $v_i \in V$ and edges $\left(v_i,v_j\right)\in E$, let $\mathbf{A} \in \mathbb{R}^{n \times n}$ denote the adjacency matrix and $\mathbf{D}$ denote the diagonal degree matrix, i.e., $\mathbf{D}_{i i}=\sum_{(v_i, v_j) \in E}\mathbf{A}_{i j}$. Let $\hat{\mathbf{A}} = \left(\mathbf{D}+\mathbf{I}_n\right)^{-\frac{1}{2}}(\mathbf{I}_n+\mathbf{A})(\mathbf{D}+\mathbf{I}_n)^{-\frac{1}{2}}$  denote a normalized variant of the adjacency matrix (with self-loops). We denote the node feature matrix as $\mathbf{X} \in \mathbb{R}^{n \times d}$ for a $d $-dimensional feature vector $\mathbf{x}_i$ associated with each node $v_i$ and use $\mathcal{N}(i)$ to represent the set of neighboring nodes of node $v_i$ (including itself). 

We next introduce message passing in graphical models. Our starting point is structured data modeled with the PGM. Markov random field is a PGM based on an undirected graph. It is especially suitable for handling high-dimensional data with graph structure. With the conditional independence assumptions of graph structure, we can simplify the joint distribution $p$ over some set of discrete random variables $\mathcal{X} = \{x_1,...,x_n\}$, and then solve it by optimization. Considering an undirected Graph $G$, we introduce latent variable ${z}_u$ associated with each node $u$. The input node features $\mathbf{X}\in \mathbb{R}^{n \times d}$ 
is the observed variable. The aim is to find conditional distributions $p\left(\mathbf{z}\mid\mathbf{X}\right)$ over latent variables $\mathbf{z} = \{{z}_1,\cdots,z_n\}$ so that we can make inference.

For a Markov random field with graph $G$, computing the posterior $p$ is a computationally intractable task. Mean-field variational inference, one of the most popular approximate inference techniques, is employed here for approximation:%
\begin{equation}
  p\left({z}_{1}, \cdots, {z}_{n} \mid \mathbf{X}\right) \approx \prod_{i=1}^{n} q_{i}\left({z}_{i}\right).
\end{equation}
By minimizing the free energy between the approximate posterior $q$ and the true posterior $p$, we obtain the optimal distribution $q^{*}$. It's hard to find a direct solution to this optimization problem. \cite{jordan} shows that the above optimization problem satisfies the following fixed point equations:%

\begin{equation*}
\log q^{*}_{i}\left(z_{i}\right)=c+\log \left(\phi\left(z_{i}, x_{i}\right)\right)+ 
\sum_{j \in \mathcal{N}(i)} 
\sum_{z_j} q^{*}_{j}\left(z_{j}\right) \log \left(\psi\left(z_{i}, z_{j}\right) \right) ,
\end{equation*}
where $c$ is a constant, $\psi$ and $\phi$ are non-negative potential functions. The above equation exhibits the relationship between marginals of adjacent nodes, and we can obtain its iterative form solution:
\begin{equation}
\begin{aligned}
\label{fixed}
\log q_{i}^{(l+1)}\left(z_{i}\right)=&c^{(l)}_i+\log \left(\phi\left(z_{i}, x_{i}\right)\right) \\&+
\sum_{j \in \mathcal{N}(i)} \sum_{z_j}
q^{(l)}_{j}\left(z_j\right)\log \left(\psi\left(z_{i}, z_{j}\right) \right),
\end{aligned}
\end{equation}
where $q_i^{*}(z_i)=\underset{l \rightarrow \infty}{\lim} q_{}^{l}(z_i)$. We can abbreviate Eq. \eqref{fixed} as:
\begin{equation}
\label{qtheta}
q_{i}^{(l+1)}\left({z}_{i} \right)=\mathcal{F}_{i} \left({z}_{i},{q^{(l)}_{j}\left(z_j\right)}\right), \quad  {j \in \mathcal{N}(i)}.
\end{equation}
Here $\mathcal{F}_{{i}}(\cdot)$ is a function determined by the potential functions. Note that these potential functions are defined on the clique where node $i$ is located. Eq. \eqref{qtheta} establishes an iterative update formula to aggregate the neighbor information. And \cite{dai} further mapped it to a high-dimensional space with the help of Hilbert space embedding, thus linking the high-dimensional data and the low-dimensional structure. They map distributions $q(z)$ into Hilbert space using some injective functions $\phi_{\alpha}$. As a result, we can obtain the embedding ${\mu}_{z_i}$ in Hilbert space by ${\mu}_{z_i}:=\int_{{z_i}} \phi_{\alpha}(z_i) q_i(z_i) d z_i$.  

\begin{theorem}\cite{smolaembedding}
Given a finite-dimensional feature map $\phi$ that maps $p(x)$ to $\mu$, if $\phi$ is injective, then any function that applies on p(x) is equivalent to computing a corresponding function on $\mu$.
\end{theorem}

According to Theorem 1, we can denote the function corresponding to $\mathcal{F}_{{i}}$ in Eq. \eqref{qtheta} after the kernel embedding as $\tilde{\mathcal{F}}_{\alpha}$. For the  recursive expression of $q_{{i}}$ (Eq. \eqref{qtheta}), we have the following iterative formula after Hilbert space embedding:
\begin{equation}
 \label{eqit}
{u}_{i}^{(l)}=\tilde{\mathcal{F}}_{\alpha,i}\left({u}_{j}^{(l-1)}, {j \in \mathcal{N}(i)}\right)  .
\end{equation}
In order to have a more straightforward distinction, we use $\ell$ to denote the number of network layers in graph networks and $l$ to represent the number of iterations.

Following the breakthrough work of \cite{dai}, several papers put forward further discussions. \cite{graphite} find the theoretical connection between message-passing GNN and mean-field variational inference. This inspires us to apply the idea to analyzing GNNs and explore it further. For simplicity of derivation, we start from the assumption that $u_i$ is unidimensional. We denote $\mathbf{U}_{i}\in \mathbb{R}^{m\times1}, m=\left|\mathcal{N}(i)\right|$, as the vector consisting of the neighbourhoods of node $i$:%
\begin{equation}
\label{eq:U}
\mathbf{U}_{i}^{(l)} =
\left[ u_1^{(l)},u_2^{(l)}, ...,u_m^{(l)}\right]^{\top}, \quad m \in \mathcal{N}(i).
\end{equation}
where $u_m^{(l)}$ is the $m$-th entrie of vector $\mathbf{U}_{i}^{(l)}$. Now we can give the Taylor expansion of Eq. \eqref{eqit} at $\mathbf{U}_{i}^{(l)}=\mathbf{0}$ as:

\begin{align}
\label{taylor}
  u_{i}^{(l)} = &\tilde{\mathcal{F}}_{\alpha, {i}}(\mathbf{0})+ \tilde{\mathcal{F}}_{\alpha, {i}}^{\prime}(\mathbf{0}) \mathbf{U}_{i}^{(l-1)}  +\frac{1}{2} {\mathbf{U}_{i}^{\left(l-1\right)^{\top}}} \tilde{\mathcal{F}}_{\alpha, {i}}^{\prime\prime}(\mathbf{0})  \mathbf{U}_{i}^{(l-1)} \nonumber\\&+\cdots+\frac{\tilde{\mathcal{F}}_{\alpha, {i}}^{n}(\mathbf{0})}{n !}\left[\mathbf{U}_{i}^{(l-1)}\right]^{n}+o\left[\mathbf{U}_{i}^{(l-1)}\right]^{n}.
\end{align}

We denote the $n^{th}$ order derivative by $\tilde{\mathcal{F}}_{\alpha, {i}}^{n}$. The value of $\tilde{\mathcal{F}}_{\alpha, {i}}^{n}(\mathbf{0})$ is node-specific. Actually, we can already find the connection between the probabilistic graph model and the GNN through Eq. \eqref{taylor}. The embedded high-dimensional variables and the iterative message-passing mechanism match perfectly with GNN. To make it more understandable, in the following, we briefly analyze the first-order approximation of Eq. \eqref{taylor}:
\begin{equation}
u_{i}^{(l)} \approx \tilde{\mathcal{F}}_{\alpha, {i}}(\mathbf{0})+ \tilde{\mathcal{F}}_{\alpha, {i}}^{\prime}(\mathbf{0}) \mathbf{U}_{i}^{\left(l-1\right)},
\label{first-order}
\end{equation}
where $\mathbf{U}_i^{(l-1)}$ consists of its neighbors and $u_{i}^{(l)}$ corresponds to the embedding of conditional probability $q_{i}$ in the Hilbert space. We can find that Eq. \eqref{first-order} satisfies the recursive message-passing mechanism. $\tilde{\mathcal{F}}_{\alpha, {i}}^{\prime}(\mathbf{0}) \mathbf{U}_{i}^{(l-1)}$ is an inner product term, which aggregates neighbors $u_{j}^{(l-1)}$ (including $u_{i}^{(l-1)}$) after multiplying by the coefficients:
\begin{equation}
\label{eq:8}
u_{i}^{(l)} \approx \tilde{\mathcal{F}}_{\alpha, {i}}(\mathbf{0})+\sum_{j \in \mathcal{N}
{(i)}} f_{ij} \cdot u_j^{(l-1)},
\end{equation}
where $[f_{ij}, j \in \mathcal{N}(i)]^{\top} \triangleq \tilde{\mathcal{F}}_{\alpha, {i}}^{\prime}(\mathbf{0})$. This form is exactly consistent with the neural message-passing mechanism of GNN.

\section{A Unified Framework}
In the previous section, we elaborate message-passing mechanism of GNN from the perspective of inference on PGMs. According to Eq. \eqref{eq:8}, we can find that the vanilla GNN is equivalent to the first-order approximation of Eq. \eqref{taylor}. In this section, we establish a unified framework to interpret existing GNN baselines and some well-designed deep GNN variants. 

 The formula of variational inference after embedding into Hilbert space exhibits a close connection with the mechanism of GNNs.  Eq. \eqref{eqit} aggregates neighborhood variables and updates its own variable, which is the same as the key core of neural message-passing. This implies that different GNN propagation mechanisms are different approximations of the function $\tilde{\mathcal{F}}_{\alpha,i}$. With the help of Taylor expansion Eq. \eqref{taylor}, we can understand this framework more clearly. In the following, we theoretically prove that the propagation mechanisms of some popular GNN variants 
 can be seen as fitting different order approximations of Eq. \eqref{taylor}.

\subsection{Interpreting GCN and SGC}
Graph Convolutional Network (GCN) is a popular baseline GNN inspired by spectral graph convolutions. We show the definition of vanilla graph convolution \cite{gcn} with bias as follows:
\begin{equation}
\label{eq:gcn}
h_{i}^{(\ell)} = \sigma\left(b^{(\ell)}_{i}+\sum_{j \in \mathcal{N}
{(i)}} \frac{w_j^{(\ell)}}{\sqrt{d_id_j}} h_{j}^{(\ell-1)} \right).
\end{equation}

A $k$-layer GCN can aggregate the messages of $k$-hop neighbors. Simplifying Graph Convolutional Network (SGC) removes nonlinearity activation in GCN while doing propagations. We will discuss the effect of nonlinearity activation in Section \ref{diss}. Here we assume that nonlinearity activations do not significantly affect the network and analyze the two similar models, SGC and GCN, together.

As a basic message-passing neural network, GCN naturally satisfies the form of Eq. \eqref{eq:8}. To be specific, $\tilde{\mathcal{F}}_{\alpha, {i}}$ is determined by the potential functions defined on the clique, but since the kernel mapping $\phi_{\alpha}$ is undetermined, the learnable weights $w$ coming with clique-related constants $d_i,d_j$ are used to approximate it. On the other hand, the bias in 
Eq. \eqref{eq:gcn} corresponds to the constant term in  Eq. \eqref{eq:8}. The specific equivalence is as follows:
\begin{equation}
h_{i}^{(\ell)}  \Rightarrow u_{i}^{(l)} , \quad 
   b_{i} \Rightarrow \tilde{\mathcal{F}}_{\alpha, {i}}(\mathbf{0}), \quad \frac{w_j^{(\ell)}}{\sqrt{d_id_j}} \Rightarrow f_{ij}. \nonumber
\end{equation}
The aggregation layer of the GCN is inherently connected to the first-order approximation of Eq. \eqref{taylor}. Through the analysis in Section \ref{pd}, we will find that such a propagation process will deteriorate the performance after the network structure is deepened.

\subsection{Interpreting GAT}

Graph attention network (GAT) \cite{gat} applies a popular attention mechanism to the aggregation layer. It uses attention weights to weigh aggregate neighbor messages. The ways to parameterize the attention weights $\alpha$ can be dot product or feedforward neural network. The layers of the GAT-styled model can be represented as follows: 
\begin{equation}
\begin{aligned}
h^{(\ell)}_{i} = \sigma\left(\sum_{j\in\mathcal{N}(i)}\alpha_{ij}w_jh_j^{(\ell-1)}\right), 
\\
\alpha_{ij} = \text{Softmax}\left(f(h_{i}^{(\ell-1)},h_{j}^{(\ell-1)})\right), 
\end{aligned}
\end{equation}
where $\alpha_{ij}$ denotes the attention weight on neighbor $j \in \mathcal{N}(i)$ when we are aggregating information at node $v_i$, and $f(\cdot)$ denotes a learnable attention weight calculation function.

We notice that the node aggregation process of GCN and GAT is quite similar. Compared with GCN, which uses the degree of nodes $d$ for weighted aggregation, GAT uses attention weights $\alpha$ instead. So its equivalence with Eq. \eqref{eq:8} is:
\begin{equation}
    \alpha_{ij}w_j^{\ell} \approx \left.\dfrac{\partial\tilde{\mathcal{F}}_{\alpha, {i}}(\mathbf{U}_i)}{\partial u_j^{}}\right|_{\mathbf{U}_i=0}.
\end{equation}
The attention mechanism has a greater computational cost, but its attention coefficients can better fit $\tilde{\mathcal{F}}_{\alpha, {i}}$ than GCN. This is because the value of the function $\tilde{\mathcal{F}}_{\alpha, {i}}$ is determined by the node $v_i$ and its neighbors. GAT can learn the  features of the clique according to the node vector and its neighbors, while GCN can only use the node degree. The experimental results in Section \ref{exp} also verified this conclusion. 

Both GCN and GAT are derived based on the first-order Taylor expansion, so they have performance degradation problems. There are several existing works on this issue, and their network architectures can alleviate the performance degradation problem of deep networks to some extent. In the following, we will interpret them from the perspective of PGM representation and analyze why they can achieve better results.

\subsection{Interpreting APPNP and GCNII}
\label{appnp/gcnii}
In the previous section, we elaborated on the classical GNNs under our framework. Intuitively, we have two ways to design a more powerful GNN. One is to find a more precise technique to fit  $\tilde{\mathcal{F}}_{\alpha, {i}}$, and the other is to perform message propagation based on the higher-order expansion of Eq. \eqref{taylor}. Indeed, the propagation process of APPNP and GCNII can be regarded as the generalized form of second-order expansion. To prove this, we first briefly present two models.

APPNP \cite{appnp} is a simple GNN originated from Personalized PageRank. It can be considered as the linear combination of the initial features and the current layer:
\begin{equation}
\mathbf{H}^{(\ell+1)}=(1-\alpha) \hat{\mathbf{A}} \mathbf{H}^{(\ell)}+\alpha \mathbf{H}^{(0)}. 
\end{equation}
GCNII \cite{gcnii} and APPNP explore the improvement of GCN from different perspectives. Still, the feature propagation of GCNII can be regarded as introducing identity mapping and a trainable weight matrix based on APPNP. The network propagation of GCNII is:
\begin{equation}
\begin{aligned}
    \mathbf{H}^{(\ell+1)}=\sigma\left(\left(\left(1-\alpha\right) \hat{\mathbf{A}} \mathbf{H}^{(\ell)}+\alpha\mathbf{H}^{(0)}\right)\right.\\
\qquad  \qquad \qquad \left.\left(\left(1-\beta\right)\mathbf{I}_{n}+\beta\mathbf{W}^{(\ell)}\right)\right).
\end{aligned}
 \label{eq:gcnii}
\end{equation}
The identity mapping in GCNII model helps assure that it performs at least as well as a shallow network. While in the theoretical analysis, we do not need to consider network optimization. In this case, we can reformulate the aggregation formula of GCNII as:
\begin{equation}
\label{gcnii-re}
\mathbf{H}^{(\ell+1)}=\sigma\left(\hat{\mathbf{A}} \mathbf{H}^{(\ell)}\mathbf{W}^{(\ell_{1})}+ \mathbf{H}^{(0)}\mathbf{W}^{(\ell_{2})}\right).
\end{equation}
When we adopt $\beta=0$ in GCNII, we can find out that GCNII reduces to APPNP \cite{appnp}. In other words, APPNP can be seen as a degenerate form of GCNII. Next, we will only analyze the propagation process of GCNII, and it is straightforward to modify the proof to obtain a nearly equivalent result for the APPNP.

\begin{proposition}\label{pro1}
The propagation process of GCNII is equivalent to the second-order expansion of Eq. \eqref{taylor}, where the initial residual corresponds to the approximation of the quadratic term.
\end{proposition}

\begin{proof}
We start with the second-order Taylor expansion:
\begin{equation}
\label{2nd-order}
\begin{aligned}
  u_{i}^{(l)} &\approx \tilde{\mathcal{F}}_{\alpha, {i}}(\mathbf{0})+ \tilde{\mathcal{F}}^{\prime}_{\alpha, {i}}(\mathbf{0}) \mathbf{U}_{i}^{(l-1)}  \\&+\frac{1}{2} \mathbf{U}_{i}^{(l-1)^{\top}}  \tilde{\mathcal{F}}_{\alpha, {i}}^{\prime\prime}(\mathbf{0})  \mathbf{U}_{i}^{(l-1)} +o\left[\mathbf{U}_{i}^{(l-1)}\right]^{2}.
 \end{aligned}
\end{equation}
We cannot directly express the quadratic term in the form of a graph convolution. To circumvent this problem, we approximate it with a linear term of the node features $\mathbf{U}_i$. We know from Eq. \eqref{eq:U} that $\mathbf{U}_i$ is a tensor consisting of the features of node $v_i$ and its neighbors. So we can rewrite the quadratic term as:
\begin{equation}
   \frac{1}{2} \mathbf{U}_{i}^{(l-1)^{\top}}  \tilde{\mathcal{F}}_{\alpha, {i}}^{\prime\prime}(\mathbf{0})  \mathbf{U}_{i}^{(l-1)} = f_1\left({u}_i^{(0)}\right) + \sum_{j\in \mathcal{N}(i)}f_2\left(u^{(l-1)}_j\right),
\end{equation}
where $f_1$ and $f_2$ denote different transformations. Hence Eq. \eqref{2nd-order} is deduced as the following form:
\begin{equation}
\label{eq:13}
u_{i}^{(l)} \approx \tilde{\mathcal{F}}_{\alpha, {i}}(\mathbf{0})+\sum_{j \in \mathcal{N}
{(i)}} g_{ij} \cdot u_j^{(l-1)}+f_1\left({u}_i^{(0)}\right),
\end{equation}
where we only keep the linear terms about $u_i$ and $u_j$, and $g_{ij}$ represent the parameters. We represent Eq. \eqref{2nd-order} as the first-order expansion plus the initial residual, which is consistent with the ideas of GCNII and APPNP. The detailed proof is deferred to the appendix.
\end{proof}
With this connection, one can show that the propagation process of GCNII is equivalent to the approximate form of the second-order expansion of Eq. \eqref{taylor}. This means that our framework based on PGM representation can well explain this deep GNN model.

\subsection{Interpreting JKNet and DGCN}
Jumping Knowledge Network (JKNet) and Decoupled Graph Convolutional Network (DGCN) are well-designed deep GNNs, and we connect their propagation operations in our framework using a similar pattern.

 The JKNet proposed by \cite{jknet} combines all previous representations at the last layer. The specific combination mechanisms can be max-pooling, concatenation, and LSTM-attention. For convenience, as is similar to \cite{ma2021unified}, we omit the non-linear activation $\sigma(\cdot)$. We denote $\mathbf{Z}$ as the final output of the model and take the attention mechanism as an example of combination process for analyzing. The attention-JKNet can be described as:
\begin{equation}
 \mathbf{Z}=\sum_{\ell=1}^L \alpha_{\ell} \hat{\mathbf{A}}^{\ell}\mathbf{X} \mathbf{W}^{\ell},
\end{equation}
 where $\alpha_{\ell}$ are learnable attention weights and $\sum_{\ell=1}^{L}\alpha_{\ell}=1$. DGCN adopts a representation ensemble similar to JKNet, and uses a transformation process in each layer similar to GCNII :
  \begin{equation}
 \mathbf{Z}=\sum_{\ell=1}^L \alpha_{\ell} \hat{\mathbf{A}}^{\ell} \mathbf{X}\left(\beta_{\ell} \mathbf{W}^{(\ell)}+\left(1-\beta_{\ell}\right) \mathbf{I}_{n}\right).
\end{equation}

$\alpha_{\ell}$ and $\beta_{\ell}$ are trainable weights, $\mathbf{I}_{n}$ is the identity mapping. DGCN ensures that the performance of the deep network is at least not worse than the shallow network by adding identity mapping. In fact, its propagation process is the same as JKNet. 

For this kind of operation that ensemble the representations of $k$ layers, it corresponds to the $k$-th expansion of Eq. \eqref{taylor}:
\begin{equation}
    u_i^{(l)} = \sum_{k = 1}^{K} \frac{\tilde{\mathcal{F}}_{\alpha, {i}}^{k}(\mathbf{0})}{k !}\left[\mathbf{U}_{i}^{(l-1)}\right]^{k},
\end{equation}
\begin{equation*}
    \frac{\tilde{\mathcal{F}}_{\alpha, {i}}^{k}(\mathbf{0})}{k !}\left[\mathbf{U}_{i}^{(l-1)}\right]^{k} \Rightarrow \alpha_{k} \hat{\mathbf{A}}^{k}\mathbf{X} \mathbf{W}^{k}.
\end{equation*}
Since the $k$-th order term of the expansion will contain high-dimensional tensors, we do not show the specific correspondence here. But we can know that the $k$-th term in Eq. \eqref{taylor} can aggregate k-hop information, and the coefficient items $\frac{\tilde{\mathcal{F}}_{\alpha, {i}}^{k}(\mathbf{0})}{k !}$ are fitted by the trainable weight matrix.

In section \ref{appnp/gcnii}, we proved that the transformations and identity mappings of GCNII can effectively reduce the error, so DGCN that applies a similar transformations method will theoretically perform better than JKNet. The experimental results also verify our presumptions.


\subsection{Discussion on Non-linearity}
\label{diss}
In the above subsections, we did not pay attention to non-linearity when analyzing these models. Here we discuss the role of activation functions in our framework. Recall what we stated in section \ref{xxx}, when we embed the formula of the obtained iterative solution into the Hilbert space, the mapping function needs to be injective. In \cite{xu2018powerful}, the authors discussed the graph convolution and proved its injectivity. This implies that adding an activation function (ReLU) satisfies our assumption. On the other hand, our network parameters need to fit the function $\tilde{\mathcal{F}}$, and the non-linear activation can help the network better fit the non-linear functions.

\section{On Designing Deep GNNs}
In section 4, we obtained a unified framework that establishes the connection between the GNN models and the Taylor expansions of Eq. \eqref{eqit}. However, the popular message-passing GNNs, such as GCN and GAT, are known to suffer from the performance degradation problem, which limits the deepening of their network structure. In this section, we use our framework to understand the performance degradation problem and explore a methodology to improve GNNs.

\subsection{Performance Degradation Problem}
\label{pd}

Several works study this performance degradation phenomenon from different perspectives, including generalization capability, low pass filtering on graphs, optimization of GNN, etc. Here we use the relationship between the GNN and Eq. \eqref{taylor} to analyze the reasons for the performance degradation of the classical GNNs.

We know that after enough iterations of Eq. \eqref{qtheta}, $q^{l}$ will converge to the optimal solution $q^{*}$. However, the propagation formula of GNN is not identical to Eq. \eqref{taylor} (Hilbert space embedding of Eq. \eqref{qtheta}), but more similar to its first-order approximation. This means each layer of the GNN introduces an extra noise $\epsilon^{(l)}$ when updating $q^{(l)}$:
 \begin{equation}
    \hat{q_{i}}^{(l+1)}\left({z}_{i} \right)+ \epsilon_{i}^{(l+1)} ={F}_{i} \left({z}_{i},\hat{q_{j}}^{(l)}\left(z_j\right)\right), \quad  {j \in \mathcal{N}(i)},
\end{equation}
where $F_i$ denotes the approximate form of Eq. \eqref{qtheta}, and $\hat{q_{i}}$ denotes the iterative result\footnote{Technically, the iteration of the GNN acts on the embedding vector of $q_i$, here we simply replace it with $q_i$.} of the GNNs.
As the network deepens (the number of iterations increases), the convergence rate of $q^{(l)}$ gradually slows down, but the accumulation of errors $\epsilon^{(l)}$ increases. Therefore, the performance of baseline GNNs degrades greatly when stacking more layers. For concreteness, we limit our discussion here. We refer readers to the appendix for a more detailed description.

\subsection{Coupling Graph Neural Network}
\label{cognet}
Based on the established framework, we have some approaches to improve the performance of the GNN model from a high-level view. Using GNN for semi-supervised node classification can be regarded as a marginal distribution estimation of latent variables on PGMs. Thus, we can improve GNN from the perspective of the iterative solution of PGM. A deeper network corresponds to more iterations in Eq. \eqref{eqit}; hence it can obtain better results. But this still requires the propagation process of the network to fit Eq. \eqref{taylor} as closely as possible. The approximation to Eq. \eqref{taylor} will introduce errors, causing performance degradation after the network reaches a certain depth.

Intuitively, Eq. \eqref{taylor} can be considered as the upper bound of the performance that variants of GNN can achieve. However, the computational cost of the propagation process and network structure which completely follows Eq. \eqref{taylor}, is unacceptable. To this end, We need to design a computationally efficient approximation of Eq. \eqref{taylor} with as little error as possible. We propose a novel GNN model, Coupling Graph Neural Network (CoGNet),  based on our established framework. Formally, we define the $\ell$-th layer of CoGNet as:
\begin{equation}
\begin{aligned}
&\mathbf{P}^{(\ell+1)}= \mathcal{G}(\mathbf{H}^{\ell},\mathbf{H}^{\ell-1})(\lambda_{\ell}\mathbf{{W}}_{\ell}+(1-\lambda_{\ell})\mathbf{I}_{n}),\\
&\mathcal{G}(\mathbf{H}^{\ell},\mathbf{H}^{\ell-1}) = \gamma_{\ell} \hat{\mathbf{A}} \mathbf{H}^{(\ell)}+(1-\gamma_{\ell})\mathbf{H}^{(\ell-1)},
\end{aligned}
\end{equation}
where $\lambda$ and $\gamma$ are layer-specific learnable weights, and $\mathbf{H}^{(\ell+1)} = \text{ReLU}(\mathbf{P}^{(\ell+1)})$. 

In fact, the propagation process of CoGNet is equivalent to the second-order Taylor expansion of Eq. \eqref{eqit}, but it is more accurate than GCNII. See the appendix for more detailed proof. We use the coupling of the two representations $\mathcal{G}(\mathbf{H}^{\ell},\mathbf{H}^{\ell-1})$ as an approximation to Eq. \eqref{taylor}, which reduces computational cost with small approximation errors. Note that in deep network layers, we use the initial representation as the coupling term. On the other hand, the learnable weights $\lambda$ and $\gamma$ enable the network to fit better the propagation process of Eq. \eqref{taylor}. In addition, following the idea of GCNII, we also introduce identity mapping to enhance the deep network performance further. 

The propagation process of CoGNet, compared with the vanilla GCN, introduces additional storage costs. Therefore, we follow the idea of the Reversible Transformer \cite{reformer} and design reversible coupling, which recovers the representation of the previous layer in the backward pass without the need for checkpoints. For the large graph input, the reversible coupling can effectively reduce the model's memory requirements.

\section{Experiment}
\label{exp}
In this section, we conduct extensive semi-supervised node classification experiments to evaluate CoGNet. We test the performance of CoGNet  on the citation networks and natural language processing (NLP) downstream tasks.

\subsection{Citation Networks}
 \subsubsection{Dataset and Baselines}  We use three standard benchmark citation datasets:  Cora, Citeseer, and Pubmed  for semi-supervised node classification. Their statistics are summarized in the appendix. We use the same fixed training/validation/testing split as \cite{gcn} on three datasets. We evaluate our model against the following baselines: 
 \begin{itemize}
     \item Graph convolutions: Chebyshev \cite{Chebyshev}, GCN \cite{gcn}, GAT \cite{gat}, SGC \cite{sgc}, APPNP \cite{appnp}, Graph U-Net \cite{unet}.
     \item Deep GNNs: JKNet \cite{jknet}, GCNII \cite{gcnii}, DAGNN \cite{dagnn}, DGCN \cite{argue}.
     \item Regularized GNN: Dropedge \cite{dropedge}, GraphMix \cite{graphmix}, GRAND \cite{grand}. Same as \cite{grand}, we report the results of these methods with GCN as the backbone.
 \end{itemize}
{We use CoGNet-S to denote a shallow CoGNet with less than 4 layers.} We further propose the CoGNet++ model, which applies two strategies, Consistency Regularization \cite{grand} and Dropedge \cite{dropedge}, on CoGNet.

\subsubsection{Experiment Results}
 We summarize the performance on citation datasets in Table \ref{tab:c}. We run CoGNet 100 times and report the mean and standard deviation of these 100 runs. 
 
 The upper part of Table \ref{tab:c} shows the variants of GNN, which are all shallow models. We can observe that CoGNet-S outperforms most baseline models. Specifically, the accuracy of the CoGNet compared with GCN and GAT has been greatly improved. In addition, we also compare the performance of popular deep GNN models, such as JKNet, GCNII, DAGNN, etc. We observe CoGNet outperforms all baseline deep GNN models. On the other hand, the deep CoGNet has a significant improvement over its shallow model, indicating that it can effectively alleviate the performance degradation problem of GNNs.

 To better compare the performance of CoGNet and other deep GNNs, we plot the performance of the deep models with various numbers of layers in Fig. \ref{fig:1}. We can observe that for the same number of network layers, CoGNet performs best in most cases. This result verifies the superiority of the propagation process of our model. It is worth noting that for the performance of shallow layers, models such as JKNet and GCNII have no advantage over GCN, but CoGNet can achieve better results than them in shallow networks. 
 
 We also apply the popular regularization strategy on graph-based models to our model, proposing CoGNet++. We can observe that it achieves state-of-the-art performance on Cora and Pubmed datasets, and it also achieves competitive results on the Citeseer dataset.
 
 
 
\begin{table}[t]
    \centering
    \begin{tabular}{c|ccc}
    \toprule
      Model   & Cora & Citeseer & Pubmed  \\
      \midrule
      Chebyshev & 81.2 & 69.8 & 74.4 \\
      GCN & 81.4 & 70.9 & 79.1 \\
      GAT & 83.3 & 72.6 & 78.5 \\
      SGC & 81.0 & 71.9 & 78.9 \\
      APPNP & 83.3 & 71.8 & 80.1 \\
      Graph U-Net & 84.4 & 73.2 & 79.6 \\
      \midrule
      JKNet & 81.1 & 69.8 & 78.1 \\
      GCNII & 85.5 & 73.4 & 80.2 \\
      DAGNN & 84.4 & 73.3 & 80.5 \\
      DGCN & 84.8 & 72.7& 80.0 \\
      \midrule
      Dropedge & 82.8 & 72.3 & 79.6 \\
      GraphMix &83.9& 74.5& 81.0\\
      GRAND & 85.4& $\mathbf{75.4}$&82.7\\
      \midrule
      CoGNet-S & $84.4 \pm0.7$ & $72.4\pm0.6$ & $79.6\pm0.9$  \\ 
      CoGNet & $85.6\pm0.5$ & $73.7\pm0.7$ & $80.5\pm0.5$  \\
      CoGNet++ & $\mathbf{86.5\pm 0.3} $ & $75.2\pm0.4$ & $\mathbf{82.9\pm0.4}$ \\
      \bottomrule
    \end{tabular}
     \caption{Summary of semi-supervised classification accuracy results on Cora, Citeseer, and Pubmed.}
    \label{tab:c}
\end{table}

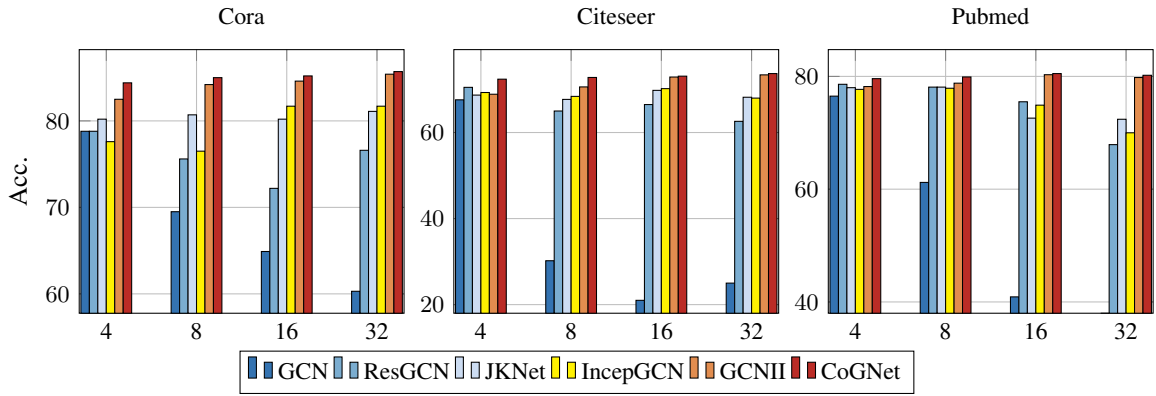
\begin{figure*}
\pgfplotsset{
tick label style={font=\footnotesize},
label style={font=\footnotesize},
legend style={font=\footnotesize},
}
\pgfplotsset{width=5.9 cm,compat=1.3}
\begin{center}
\begin{tikzpicture}
\begin{axis}[
title = \footnotesize{Cora},
grid = major,
width = 5.9cm,
ylabel=Acc.,
ybar=0pt,
symbolic x coords={4,8,16,32},
xtick=data,
bar width=3.2pt,
legend columns=-1,
legend entries={GCN,ResGCN,JKNet,IncepGCN,GCNII,CoGNet},
legend to name = named,
xtick align = inside,
]
\addplot[fill=color1] coordinates {
(4,78.8) (8,69.5)
(16,64.9) (32,60.3) 
};
\addplot[fill=color2] coordinates {
(4,78.8) (8,75.6)
(16,72.2) (32,76.6) 
};
\addplot[fill=color3] coordinates {
(4,80.2) (8,80.7)
(16,80.2) (32,81.1) 
};
\addplot[fill=yellow] coordinates {
(4,77.6) (8,76.5)
(16,81.7) (32,81.7) 
};
\addplot[fill=color5] coordinates {
(4,82.5) (8,84.2)
(16,84.6) (32,85.4) 
};
\addplot[fill=color6] coordinates {
(4,84.4) (8,85.0)
(16,85.2) (32,85.7) 
};
\end{axis}
\end{tikzpicture} 
\begin{tikzpicture}
\begin{axis}[
title = \footnotesize{Citeseer},
width = 5.9cm,
grid=major,
ybar=0pt,
symbolic x coords={4,8,16,32},
xtick=data,
bar width=3.2pt,
ymin=18,
xtick align = inside,
]
\addplot[fill=color1] coordinates {
(4,67.6) (8,30.2)
(16,21) (32,25) 
};
\addplot[fill=color2] coordinates {
(4,70.5) (8,65)
(16,66.5) (32,62.6) 
};
\addplot[fill=color3] coordinates {
(4,68.7) (8,67.7)
(16,69.8) (32,68.2) 
};
\addplot[fill=yellow] coordinates {
(4,69.3) (8,68.4)
(16,70.2) (32,68) 
};
\addplot[fill=color5] coordinates {
(4,68.9) (8,70.6)
(16,72.9) (32,73.4) 
};
\addplot[fill=color6] coordinates {
(4,72.4) (8,72.8)
(16,73.1) (32,73.7) 
};
\end{axis}
\end{tikzpicture} 
\begin{tikzpicture}
\begin{axis}[
title = \footnotesize{Pubmed},
width = 5.9cm,
grid=major,
ybar=0pt,
symbolic x coords={4,8,16,32},
xtick=data,
bar width=3.2pt,
ymin=38,
xtick align = inside,
]
\addplot[fill=color1] coordinates {
(4,76.5) (8,61.2)
(16,40.9) (32,38) 
};
\addplot[fill=color2] coordinates {
(4,78.6) (8,78.1)
(16,75.5) (32,67.9) 
};
\addplot[fill=color3] coordinates {
(4,78.0) (8,78.1)
(16,72.6) (32,72.4) 
};
\addplot[fill=yellow] coordinates {
(4,77.7) (8,77.9)
(16,74.9) (32,70) 
};
\addplot[fill=color5] coordinates {
(4,78.2) (8,78.8)
(16,80.3) (32,79.8) 
};
\addplot[fill=color6] coordinates {
(4,79.6) (8,79.9)
(16,80.5) (32,80.2) 
};
\end{axis}
\end{tikzpicture} \\
\ref{named}
\end{center}
\caption{Semi-supervised node classification performance with various depths from 4 to 32.}
\label{fig:1}
\end{figure*}

\subsection{NLP Tasks}
Many NLP problems can be naturally expressed as graph structures, such as constructing sentences as graphs with words as nodes and constructing text as graphs composed of paragraphs and entities. 
Therefore, neural graph-based models have emerged in recent years and have received increasing attention. However, shallow networks limit their performance. For example, in multi-hop Question Answering (QA), the two-layer graph convolutional network makes the nodes only access two-hop information and invisible for long-range knowledge. We only need to replace the GNN in the model \cite{textgcn,dfgn} with the CoGNet, which makes the network acquire a wider node perception field while deepening network layers and thus reach better results. We conduct experiments to test the performance of CoGNet on two NLP tasks, including text classification and multi-hop QA. Specific parameters of the experiment and more detailed illustrations refer to the appendix.

\subsubsection{Datasets}
To thoroughly compare our model with the GCN, we test it on five widely used datasets of text classification: 20NG, MR, Ohsumed, R8, and R52. Their statistics are summarized in the appendix. HotpotQA is a dataset for explainable multi-hop QA. 

\begin{table}[h]
    \centering
    \begin{tabular}{l|ccccc}
    \toprule
      Model   & 20NG & R8 &R52&Ohsumed&MR  \\
      \midrule
      TextGCN  & 86.3&97.1&93.5&68.4&76.8\\
     +SGC& \textbf{88.5}& 97.2 & 94.0 & 68.5&  75.9 \\
      +CoGNet& {88.3}& \textbf{97.4}& \textbf{94.1}& \textbf{69.1}&  \textbf{77.2} \\
      \bottomrule
    \end{tabular}
    \caption{Summary of test accuracy on text classification.}
    \label{tab:nlp}
\end{table}

\subsubsection{Text Classification}
Text classification is a fundamental task in NLP. \cite{textgcn} constructs a corpus-level heterogeneous graph containing document nodes and word nodes and uses a 2-layer GCN for classification. The results in Table \ref{tab:nlp} demonstrate a significant improvement of CoGNet over GCN and SGC in text classification.

\subsubsection{Multi-hop QA}
Multi-hop text-based QA is a popular topic in NLP. Many models explicitly find the inference path of the question with the help of graphs. For this case, multi-layer GNNs can obtain multi-hop information to assist inference. DFGN \cite{dfgn} constructs a graph of the entities mentioned in the query and then dynamically combines them with the help of a 
GAT. The two-layer GAT in DFGN can only reach two hops of information. Deepening the network depth can obtain multi-hop information, but the accuracy of downstream tasks is affected by the deep GNN. Table \ref{tab:qa} reveals the test performance on the Hotpot QA dataset. We can observe that GCN has the worst performance, while shallow CoGNet (CoGNet-S) has a significant improvement over DFGN (GAT). On the other hand, deepening the CoGNet can achieve better results.

\begin{table}[t]
\centering
\begin{tabular}{l|cccc}
\toprule Dataset & Model &  Joint EM & Joint F$_{1}$ \\
\midrule
 \multirow{4}{*}{HotpotQA} & DFGN & 33.62& 59.82 \\
 & GCN & 33.60& 59.70 \\
 & CoGNet-S &{33.75} &{60.11} \\
  & CoGNet &\textbf{34.12} &\textbf{60.15} \\
\bottomrule
\end{tabular}
\caption{Summary of test performance on multi-hop QA.}
\label{tab:qa}
\end{table}
\section{Conclusion and Future Work}
In this paper, we developed a unified theoretical framework to understand the GNN baselines as well as various deep GNN models using a graphical model representation. Specifically, we obtained an iterative solution of variational inference on Markov random fields, and the propagation operation of GNNs can be represented as approximate forms of it. We also proposed a theoretically motivated and powerful GNN which performs well on both shallow and deep network layers. An interesting direction for future work is to establish the connection between the approximate sampling methods, and the graph neural network to pursue a faster and more powerful sample-based GNN \cite{fastgcn,sage,hasanzadeh2020bayesian}. To complete the picture, understanding and improving the general GNN with the help of other variational methods would also be interesting.

\bibliographystyle{named}
\bibliography{ijcai23}

\end{document}